\documentclass[11pt,letterpaper]{article}

\usepackage[T1]{fontenc}
\usepackage[utf8]{inputenc}
\usepackage{lmodern}
\usepackage[margin=1in]{geometry}
\usepackage[round,authoryear]{natbib}
\usepackage{amsmath,amssymb,amsfonts,bm,mathtools,mathrsfs}
\usepackage{amsthm}
\usepackage[dvipsnames]{xcolor}
\usepackage{graphicx}
\usepackage{booktabs}
\usepackage{microtype}
\usepackage{nicefrac}
\usepackage{enumitem}
\usepackage{xurl}
\usepackage[colorlinks=true,allcolors=MidnightBlue]{hyperref}
\usepackage[capitalize,nameinlink]{cleveref}

\allowdisplaybreaks

\newcommand{\prn}[1]{\left(#1\right)}

\newcommand{\wtilde}[1]{\widetilde{#1}}

\def\mumin{\mu_{\min}}

\newcommand{\wtO}{\widetilde O}

\hypersetup{
  pdftitle={Hitting-Time Bounds for Asynchronous Q-Learning},
  pdfauthor={Yang Peng}
}

\newtheorem{theorem}{Theorem}[section]
\newtheorem{lemma}[theorem]{Lemma}
\newtheorem{proposition}[theorem]{Proposition}
\newtheorem{corollary}[theorem]{Corollary}

\theoremstyle{remark}

\newtheorem{example}[theorem]{Example}

\title{Hitting-Time Bounds for Asynchronous Q-Learning}
\author{Yang Peng\\
Yau Mathematical Sciences Center, Tsinghua University\\
\texttt{yang-peng@mail.tsinghua.edu.cn}}
\date{}

\begin{document}
\maketitle

\begin{abstract}
We study the last iterate of standard asynchronous Q-learning from a single
trajectory of a finite Markov decision process. The trajectory follows a
fixed behavior policy whose state--action chain is irreducible and may be
periodic. For discount factor $\gamma$, write $H=(1-\gamma)^{-1}$, and let
$\mumin$ denote the minimum stationary state--action probability and
$\mathfrak h_\star$ the largest expected hitting time of the reverse
state--action chain. We prove that zero-initialized Q-learning with a
constant step size achieves sup-norm error at most $\varepsilon$ with high
probability using
\[
 \wtilde O\prn{\frac{H^4}{\mumin\varepsilon^2}+\mathfrak h_\star}
\]
transitions, for $0<\varepsilon\le1$. The first term retains the quartic
statistical horizon dependence of standard Q-learning. The second controls
initialization error and is additive in $\mathfrak h_\star$. For aperiodic
behavior chains, a comparison with mixing time gives the corresponding
bound with transient $\wtilde O(t_{\mathrm{mix}}/\mumin)$. An empirical
state--action count rule selects a constant step for a subsequent training
run and attains the same rate without prior knowledge of unknown model
parameters, target accuracy, or confidence level. Its output is the last
training iterate, and its sample count includes both stages. The proof
uses two deterministic Bellman envelopes to separate statistical error
from fixed-policy initialization error, and adapts the reverse-time
hitting-cost compensation used for TD learning to state--action
propagation. An exploration lower bound holds for arbitrary predictable
step sizes. Together with the existing statistical lower bound and its
extension to pilot-selected constant steps, it yields matching mean-square
rates up to logarithmic factors on the stated large-sample parameter
classes.
\end{abstract}

\section{Introduction}
\label{sec:intro}
Q-learning estimates the optimal action-value function of a Markov decision
process by updating one coordinate after each observed transition.
In its standard asynchronous form, it uses a single trajectory generated
by a behavior policy and stores only the current value table.
Understanding the sample complexity of this recursion requires accounting
for both statistical error and the time needed to visit the state--action
pairs that must be learned.

Under independent synchronous sampling, the sharp statistical dependence
of standard Q-learning on the effective horizon $H=(1-\gamma)^{-1}$ is
quartic \citep{li2024q}. For a single Markov trajectory, existing bounds
also depend on the minimum stationary state--action probability
$\mumin$ and the mixing time $t_{\mathrm{mix}}$ of the behavior chain.
In particular, \citet[Theorem~4]{li2024q} gives the sufficient transition
count
\[
 \wtO\!\left(\frac{H^4}{\mumin\varepsilon^2}
                +\frac{H t_{\mathrm{mix}}}{\mumin}\right).
\]
The statistical term agrees with their algorithm-specific lower bound
in its stated parameter range. The transient term leaves open whether
the time needed to explore the behavior chain must be multiplied by
the effective horizon.

We prove that zero-initialized asynchronous Q-learning retains the
quartic statistical term with an additive transient governed by hitting
times. Specifically, for $0<\varepsilon\le1$, a suitable constant step
size gives last-iterate sup-norm accuracy $\varepsilon$ with high
probability using
\begin{equation}\label{eq:intro-rate}
 \wtO\!\left(\frac{H^4}{\mumin\varepsilon^2}
                       +\mathfrak h_\star\right)
\end{equation}
transitions. Here $\mathfrak h_\star$ is the largest expected hitting
time between two coordinates of the reverse behavior state--action
chain. Our bounds hold for any initial state--action distribution and
for finite irreducible behavior chains, including non-reversible and
periodic chains. For aperiodic chains, the comparison
$\mathfrak h_\star=\wtO(t_{\mathrm{mix}}/\mumin)$ gives the corresponding
mixing-time bound with no extra factor of $H$ in the transient.
The hitting-time formulation can also be substantially smaller for
some chains, as Example~\ref{ex:periodic} illustrates.

The proof separates stochastic fluctuations from initialization error.
Two deterministic Bellman envelopes control the sampling error using
concentration on the visit clock of each coordinate. The upper envelope
is a perturbed optimal Bellman fixed point. The lower envelope compares
Q-learning with evaluation under a fixed optimal policy and leaves a
homogeneous propagation matrix as the initialization remainder.
We represent its remaining mass by a reverse-time particle and compensate
successive hitting costs with a physical-state potential. This gives the
additive $\mathfrak h_\star$ term in \eqref{eq:intro-rate}. The envelopes
and auxiliary evaluation process are used only in the proof; the
algorithm remains standard Q-learning.

An empirical state--action count rule attains the same rate without
prior knowledge of unknown model parameters. For a prescribed budget,
the first half of the trajectory supplies visit counts, which determine
a constant step size for a zero-initialized run on the second half.
The rule uses the known discount factor, the table size, the budget,
and observed counts, but no target accuracy or confidence level.
Theorem~\ref{thm:calibrated-upper} gives a high-probability guarantee
for the last training iterate, counting both stages.
Corollary~\ref{cor:calibrated_mse} gives the corresponding mean-square
bound.

Theorem~\ref{thm:joint-lower} gives a matching lower bound, up to
logarithmic factors, for the statistical and hitting-time terms on the
parameter classes specified there. The statistical term extends the
constant-step lower bound of \citet{li2024q} to steps selected from a
pilot trajectory and then held fixed. It concerns mean-square error
in the stated large-sample range. The exploration term follows from
an unvisited-coordinate construction and holds for arbitrary
predictable step sizes.

\subsection{Related work}

Classical convergence results for asynchronous Q-learning include
\citet{tsitsiklis1994asynchronous}. Finite-time analyses of asynchronous
stochastic approximation are developed by \citet{qu2020finite}.
\citet{li2022async} obtain additive mixing-time bounds for standard and
variance-reduced Q-learning. Their Section~3.4 also gives a step-size
rule based on empirical occupancy and a guarantee for a stored
schedule-boundary iterate. Our calibration result uses a pilot stage
and a known budget to obtain \eqref{eq:intro-rate} for the last iterate
of a subsequent constant-step run.

\citet{li2024q} establish the quartic statistical rate of standard
Q-learning and its algorithm-specific lower bound. We retain this
statistical dependence and sharpen the trajectory transient to an
additive reverse hitting time. Variance reduction improves the
statistical horizon dependence \citep{li2022async}; the results here
concern the standard update.

Monotone comparison arguments for Q-learning appear in
\citet{wainwright2019cone}. \citet{lee2026sign} separates positive and
negative error components and compares the negative component with an
optimal-policy linear system. We use this type of comparison with fixed,
variance-sensitive Bellman envelopes and a homogeneous fixed-policy
propagation remainder.

\subsection{Relation to the analysis of TD learning}
\label{sec:relation-td}

\citet{peng2026td} studies importance-weighted asynchronous TD learning
for fixed-policy evaluation and obtains the transition bound
\[
 \wtO\!\left(\frac{\varrho H^3}{\mu_{\min}\varepsilon^2}
                         +\mathfrak h_\star\right),
\]
where $\varrho$ bounds the action importance ratios. In its state-value
formulation, $\mu_{\min}$ and $\mathfrak h_\star$ refer to the behavior
\emph{state} chain. That work also covers importance-weighted action-value
prediction. Here these parameters refer to the behavior
\emph{state--action} chain. Standard Q-learning uses no importance ratios:
conditional on the current state--action pair, the observed transition
has the required transition law, and the Bellman backup maximizes over
the actions at the next state.

Both analyses separate stochastic error from initialization and use a
reverse-time particle with a hitting-time potential to obtain an
additive transient. The stochastic arguments differ. \citet{peng2026td}
retains the exact affine stochastic convolution and controls it with an
anchored local Poisson equation and a return-variance certificate.
The maximization in Q-learning makes its update nonlinear. We instead
use deterministic Bellman envelopes and visit-clock concentration to
reduce the error to fixed-policy propagation. A variance certificate
valid for every fixed policy controls the upper envelope, whose greedy
policy need not be optimal for the original MDP. The remaining
propagation selects a fixed target action at the next state, in place
of the importance-weighted sampled next action in the TD action-value
formulation. Section~\ref{sec:initialization} adapts the hitting-cost
compensation to this action projection. Thus the two papers share the
initialization analysis principle, while using different estimates for
stochastic error. The resulting statistical dependence is cubic in
$H$ for TD learning and quartic for standard Q-learning.

The step-size and lower-bound results have different scopes as well.
\citet{peng2026td} gives constant and decreasing step sizes, together
with an empirical-count rule that carries the estimate across epochs
and has a time-uniform guarantee. Our empirical rule uses a prescribed
budget and a new zero-initialized training run with a fixed step.
The TD lower bounds are minimax bounds for policy evaluation on the
specified parameter classes. The quartic statistical lower bound here
is specific to the stated class of Q-learning procedures and to
mean-square error.

\begin{table}[t]
\centering\small
\caption{Single-trajectory transition bounds, up to logarithmic factors,
for standard Q-learning and $0<\varepsilon\le1$. The present upper bounds
use zero initialization. All chain parameters in this table refer to the
behavior state--action chain.}
\label{tab:comparison}
\begin{tabular}{@{}p{0.45\linewidth}p{0.49\linewidth}@{}}
\toprule
Method or guarantee & Sufficient count or statistical lower bound\\
\midrule
Constant step, \citet{li2022async} & $H^5/(\mumin\varepsilon^2)+Ht_{\mathrm{mix}}/\mumin$\\[3pt]
Constant step, \citet{li2024q} & $H^4/(\mumin\varepsilon^2)+Ht_{\mathrm{mix}}/\mumin$\\[3pt]
Constant step, this paper & $H^4/(\mumin\varepsilon^2)+\mathfrak h_\star$\\[3pt]
Empirically calibrated step, this paper & $H^4/(\mumin\varepsilon^2)+\mathfrak h_\star$\\[3pt]
Statistical lower bound, \citet{li2024q}, and its pilot extension here &
$\mathbb E\|Q_T-Q^*\|_\infty^2\gtrsim H^4/(\mumin T\log^3T)$,
in the stated large-sample range\\
\bottomrule
\end{tabular}
\end{table}

\paragraph{Organization.}
Section~\ref{sec:setup} introduces the model and Q-learning update.
Section~\ref{sec:results} states the upper bounds and calibration rule.
Sections~\ref{sec:envelopes} and~\ref{sec:initialization} develop the
Bellman envelope and initialization arguments.
Section~\ref{sec:lower-bounds} gives the lower bounds, and
Section~\ref{sec:discussion} discusses further directions.
The appendices collect occupancy concentration and the proofs that
combine these ingredients into the stated upper bounds.

\section{Problem Formulation}
\label{sec:setup}
Let $\mathcal S$ be a finite state space, let $\mathcal A(s)$ be its nonempty finite action sets, and write $\mathcal X=\{(s,a):s\in\mathcal S,a\in\mathcal A(s)\}$ and $d=|\mathcal X|$. Given $X_t=(S_t,A_t)=x$, the pair $(R_t,S_{t+1})$ has a fixed probability kernel $\mathcal K_x$ on $[0,1]\times\mathcal S$, conditionally independent of the preceding history. Rewards and successor states may be dependent within a transition. Write $P(v\mid x)$ for the state marginal and $r(x)$ for the reward mean. The discount $\gamma\in(0,1)$ is known and $H=(1-\gamma)^{-1}$.

A fixed behavior policy $b$ draws $A_t$ from $b(\cdot\mid S_t)$. Its state--action transition matrix is
\[
 K_b((s,a),(v,c))=P(v\mid s,a)b(c\mid v).
\]
We assume $K_b$ is irreducible on $\mathcal X$, and denote its unique stationary distribution by $\mu$. In particular $b(a\mid s)>0$ for each admissible action. Let
\begin{equation}\label{eq:parameters}
 \mu_{\min}=\min_{x\in\mathcal X}\mu(x),\qquad
 K_b^*(x,y)=\frac{\mu(y)K_b(y,x)}{\mu(x)},\qquad
 \mathfrak h_\star=\max_{x,y}\mathbb E_x^{K_b^*}\inf\{t\ge0:X_t=y\}.
\end{equation}
The initial physical state has any distribution. Upper bounds also allow an arbitrary initial state--action coordinate. No stationary initialization is assumed.

For a stationary target policy $\pi$, write $K_\pi((s,a),(v,c))=P(v\mid s,a)\pi(c\mid v)$ and $\mathcal T_\pi q=r+\gamma K_\pi q$. The optimal Bellman operator is
\[
 (\mathcal Tq)(s,a)=r(s,a)+\gamma\sum_vP(v\mid s,a)\max_{c\in\mathcal A(v)}q(v,c).
\]
Its unique fixed point is $Q^*$, with $0\le Q^*\le H\mathbf1$. Fix a deterministic optimal policy $\pi^*$; it exists because the MDP is finite. Sup norms of vectors and induced sup norms of matrices are both written $\|\cdot\|_\infty$.

The algorithm considered throughout is
\begin{equation}\label{eq:q-update}
 Q_{t+1}(x)=Q_t(x)+\eta_t\mathbf1_{\{X_t=x\}}
 \left[R_t+\gamma\max_{a\in\mathcal A(S_{t+1})}Q_t(S_{t+1},a)-Q_t(x)\right],
 \qquad Q_0=0.
\end{equation}
For $0\le\eta_t\le1$, the table stays in $[0,H]^d$ pathwise. The output is the last iterate; there is no averaging or variance reduction. The calibrated version starts \eqref{eq:q-update} from zero after a preliminary observation period, on the continuation of the same trajectory.

All logarithms are natural unless a base is displayed. Constants denoted by $c,C$ are positive numerical constants. The notation $\wtO$ suppresses logarithmic factors in the displayed statistical parameters, dimension, and confidence. Statements involving the statistical lower bound retain the explicit restrictions in Section~\ref{sec:lower-bounds}.

\section{Main Results}
\label{sec:results}
\subsection{A constant step size}

\begin{theorem}[Constant-step last iterate]\label{thm:constant-main}
Under the model in Section~\ref{sec:setup}, fix a deterministic $T\ge1$, $0<\varepsilon\le1$, and $0<\delta<1$. Set
\[
 L_T=\log\frac{64dH(T+1)}{\delta\varepsilon\mu_{\min}},\qquad
 \eta=\frac{\varepsilon^2}{4096H^3L_T}.
\]
There is a numerical constant $C$ such that, if
\begin{equation}\label{eq:constant-samples}
 T\ge C\left(\frac{H^4L_T}{\mu_{\min}\varepsilon^2}
                           +\mathfrak h_\star\right)L_T,
\end{equation}
then ordinary Q-learning \eqref{eq:q-update} with $\eta_t=\eta$ satisfies
\[
 \mathbb P\!\left(\|Q_T-Q^*\|_\infty\le\varepsilon\right)\ge1-\delta.
\]
\end{theorem}

The step size in Theorem~\ref{thm:constant-main} uses the target accuracy, confidence level, and minimum stationary probability. Section~\ref{sec:calibration} gives an empirical rule that does not use these quantities. The implicit appearance of $T$ inside $L_T$ is logarithmic and gives the transition count \eqref{eq:intro-rate}.

The proof is given in Appendix~\ref{app:main-proof}.

\paragraph{Proof overview.}
The proof has two inputs. The deterministic envelope argument in Section~\ref{sec:envelopes} controls the statistical term using concentration only on the visit clock of each coordinate. The propagation argument in Section~\ref{sec:initialization} bounds initialization by reversing the behavior trajectory. Together they give $\|Q_T-Q^*\|_\infty\le\varepsilon/2+H\|\Phi_{T,0}\|_\infty$, where $\Phi_{T,0}$ is the homogeneous propagator under a fixed optimal policy.

\subsection{Mixing-time consequence and periodic examples}
For an aperiodic $K_b$, define $t_{\mathrm{mix}}$ at total-variation threshold $1/4$. Total-variation contraction gives an integer
\[
 \tau\le C t_{\mathrm{mix}}\log(2/\mu_{\min})
 \quad\text{such that}\quad K_b^\tau(i,x)\ge\mu(x)/2
 \quad\text{for all }i,x.
\]
Indeed, iterate the usual contraction bound until the total-variation distance is at most $\mu_{\min}/2$. Reversal then gives
$(K_b^*)^\tau(x,i)=\mu(i)K_b^\tau(i,x)/\mu(x)\ge\mu(i)/2$.
At successive $\tau$-spaced observation times, the conditional chance of hitting $i$ is at least $\mu(i)/2$. A geometric-tail bound therefore implies
\begin{equation}\label{eq:hitting-mixing}
 \mathfrak h_\star\le\frac{2\tau}{\mu_{\min}}
 \le\frac{C t_{\mathrm{mix}}}{\mu_{\min}}\log\frac2{\mu_{\min}}.
\end{equation}
Consequently Theorem~\ref{thm:constant-main}, and the calibrated theorem below, imply
\[
 \wtO\!\left(\frac{H^4}{\mu_{\min}\varepsilon^2}
                         +\frac{t_{\mathrm{mix}}}{\mu_{\min}}\right).
\]
The comparison uses the mixing time of the forward behavior chain.

\begin{example}[Nearly alternating behavior]\label{ex:periodic}
Take two physical states, two identical actions at each state, a uniform behavior policy, and physical transition matrix
\[
 P_\alpha=\begin{pmatrix}\alpha&1-\alpha\\1-\alpha&\alpha\end{pmatrix},
 \qquad 0\le\alpha\le1/4.
\]
The state--action chain has $\mu_{\min}=1/4$. Its reverse equals itself. For a target coordinate, the expected hitting time from the other action at the target physical state is $4$; from the opposite physical state it is $2+(1-\alpha)^{-1}\le10/3$. Hence $\mathfrak h_\star=4$ throughout this family. For $\alpha>0$, the physical eigenvalue $2\alpha-1$ gives $t_{\mathrm{mix}}=\Theta(1/\alpha)$ as $\alpha\downarrow0$. At $\alpha=0$ the chain is irreducible and periodic. Thus the hitting-time bound can remain bounded as the mixing-time bound diverges, and it still applies at the periodic limit.
\end{example}

\subsection{Calibration without unknown problem parameters}
\label{sec:calibration}

The preceding guarantee uses a step size tuned to a target accuracy.
We next give a fixed-budget rule that needs no transition probabilities,
stationary probabilities, hitting times, target accuracy, or confidence
level. The discount factor and the finite state--action space remain
known inputs. The only extra operation is counting visits during an
initial calibration segment.

Fix an integer $n\ge1$, for a total budget of $2n$ transitions. Observe
the first $n$ transitions under the same behavior policy and set
\begin{equation}\label{eq:calibration_step}
 \begin{split}
  N_n(x)&=\sum_{t=0}^{n-1}\mathbf1_{\{X_t=x\}},\qquad
  M_n=\max\{1,\min_xN_n(x)\},\\
  \ell_n&=\log\bigl(16dH(n+8)^2\bigr),\qquad
  \widehat\eta_n=\min\left\{H^{-1},\frac{AH\ell_n}{M_n}\right\}.
 \end{split}
\end{equation}
Here $A=2^{18}$ is fixed once and for all. At time $n$, initialize
$Q^{(n)}_0=0$ and perform
exactly $n$ ordinary Q-learning updates on the next $n$ transitions,
using the constant step $\widehat\eta_n$. Return $Q^{(n)}_n$.
The output is the last iterate of this Q-learning run. The calibration
data determine its step size; the estimation phase uses each of its
transitions once, in chronological order.

\begin{theorem}[Calibration without unknown problem parameters]
\label{thm:calibrated-upper}
Under the assumptions of the main upper-bound theorem, let
$\widehat Q_{2n}=Q^{(n)}_n$ denote the output of
\eqref{eq:calibration_step}. There are universal constants $C_0,C_1>0$
such that the following holds. For every $\delta\in(0,1)$ and
$\varepsilon\in(0,1]$, put
\[
 L_n=\log\frac{C_0dH(n+8)^2}{\delta\mu_{\min}}.
\]
If
\begin{equation}\label{eq:calibrated_sample}
 n\ge C_1\left\{
       \frac{H^4\ell_nL_n}{\mu_{\min}\varepsilon^2}
       +\mathfrak h_\star L_n\right\},
\end{equation}
then, under every initial state--action distribution,
\[
 \mathbb P\bigl(\|\widehat Q_{2n}-Q^*\|_\infty
                        \le\varepsilon\bigr)\ge1-\delta.
\]
Thus its total transition complexity is
$\widetilde O(H^4/(\mu_{\min}\varepsilon^2)+\mathfrak h_\star)$.
The quantities $\varepsilon$, $\delta$, $\mu_{\min}$, and
$\mathfrak h_\star$ occur in the guarantee but not in the step rule.
\end{theorem}

The budget is fixed before observing the trajectory. Each budget specifies
a pilot segment and a subsequent zero-initialized training run.

The proof is given in Appendix~\ref{app:unknown-mu}; it uses the occupancy estimate in Appendix~\ref{app:coverage} and the higher-power propagation bound in Section~\ref{sec:initialization}.

\begin{corollary}[Mean-square accuracy with the same calibrated step]
\label{cor:calibrated_mse}
For $0<\varepsilon\le1$, the same algorithm satisfies
\[
 \mathbb E\|\widehat Q_{2n}-Q^*\|_\infty^2\le\varepsilon^2
\]
whenever
\[
 n\ge C\left\{
 \frac{H^4\ell_n}{\mu_{\min}\varepsilon^2}
       \log\frac{C dH^3(n+8)^2}{\mu_{\min}\varepsilon^2}
 +\mathfrak h_\star
       \log\frac{C dH^3(n+8)^2}{\mu_{\min}\varepsilon^2}
 \right\}.
\]
\end{corollary}

The proof is also given in Appendix~\ref{app:unknown-mu}.

\section{Deterministic Bellman Envelopes}
\label{sec:envelopes}
This section separates the statistical error of Q-learning from
homogeneous propagation under a fixed optimal policy. Unlike the affine
stochastic-convolution analysis of \citet{peng2026td}, the argument uses
two deterministic Bellman envelopes and concentration on the visit clock
of each coordinate. Recall the notation of Section~\ref{sec:setup},
and write $V^*(s)=\max_a Q^*(s,a)$.
Define the nonnegative function $\sigma_*$ by
\begin{equation}
 \sigma_*^2(x)
 =\operatorname{Var}(R_t+\gamma V^*(S_{t+1})\mid X_t=x).
 \label{eq:envelope-variance}
\end{equation}
All vector inequalities below are componentwise.

\begin{lemma}[A variance certificate for every fixed policy]
\label{lem:universal-variance}
The function $u=2H Q^*-(Q^*)^2$, where the square is componentwise,
satisfies
\begin{equation}
 0\le u\le H^2\mathbf 1,
 \qquad \sigma_*^2+\gamma^2 K_\pi u\le u
 \quad\text{for every stationary policy }\pi.
 \label{eq:universal-certificate}
\end{equation}
Consequently,
\begin{equation}
 (I-\gamma K_\pi)^{-1}\sigma_*^2\le 2H^2\mathbf 1,
 \qquad
 \bigl\|(I-\gamma K_\pi)^{-1}\sigma_*\bigr\|_\infty
 \le \sqrt{2}\,H^{3/2}.
 \label{eq:universal-resolvent}
\end{equation}
\end{lemma}

\begin{proof}
The function $f(z)=2Hz-z^2$ is increasing on $[0,H]$.
Fix $x=(s,a)$, write $q=Q^*(x)$, and set $v'=V^*(S_{t+1})$.
Since $Q^*(s',a')\le V^*(s')$,
\[
 (K_\pi u)(x)\le \mathbb E[2Hv'-(v')^2\mid x].
\]
Using $q=\mathbb E[R_t+\gamma v'\mid x]$ gives
\begin{align*}
 u(x)-\sigma_*^2(x)-\gamma^2(K_\pi u)(x)
 &\ge
 \mathbb E\!\left[
 R_t(2H-R_t-2\gamma v')+2\gamma v'\mid x
 \right]\ge0.
\end{align*}
Indeed, $2H-R_t-2\gamma v'\ge2H-1-2\gamma H=1$.
This argument permits arbitrary dependence between $R_t$ and
$S_{t+1}$ conditional on $x$.
Let $\mathcal R_\pi=(I-\gamma K_\pi)^{-1}$, which is nonnegative.
Then
\[
 \mathcal R_\pi\sigma_*^2
 \le\mathcal R_\pi(I-\gamma^2K_\pi)u
 =u+\gamma(1-\gamma)\mathcal R_\pi K_\pi u
 \le(1+\gamma)H^2\mathbf 1
 \le2H^2\mathbf 1.
\]
Every row sum of $\mathcal R_\pi$ equals $H$.
Applying Cauchy--Schwarz to each row yields
$(\mathcal R_\pi\sigma_*)^2
\le H\mathcal R_\pi\sigma_*^2\le2H^3\mathbf 1$.
\end{proof}

\begin{lemma}[Construction of two fixed envelopes]
\label{lem:fixed-envelopes}
Fix $0<\varepsilon\le1$, $L\ge1$, and
\begin{equation}
 0<\eta\le \frac{\varepsilon^2}{4096H^3L}.
 \label{eq:envelope-step-condition}
\end{equation}
Define
\begin{equation}
 b(x)=4\sqrt{\eta L}\bigl(\sigma_*(x)+\varepsilon/2\bigr)
      +8\eta HL.
 \label{eq:envelope-bonus}
\end{equation}
There are unique vectors $Q^+$ and $q^-$ satisfying
\begin{equation}
 Q^+=\mathcal TQ^++b,
 \qquad q^-=r+\gamma K_{\pi^*}q^--b.
 \label{eq:two-fixed-points}
\end{equation}
They obey
\begin{equation}
 Q^*-\frac{\varepsilon}{2}\mathbf 1
 \le q^-\le Q^*\le Q^+
 \le Q^*+\frac{\varepsilon}{2}\mathbf 1.
 \label{eq:envelope-closeness}
\end{equation}
Moreover, their fixed-target means satisfy
\begin{equation}
 m^+:=\mathcal TQ^+=Q^+-b\ge0,
 \qquad
 m^-:=r+\gamma K_{\pi^*}q^-=q^-+b\le Q^*\le H\mathbf 1.
 \label{eq:envelope-means}
\end{equation}
\end{lemma}

\begin{proof}
Both maps in \eqref{eq:two-fixed-points} are $\gamma$-contractions
in the maximum norm.  Monotonicity gives $Q^+\ge Q^*$, while
\[
 Q^*-q^-=(I-\gamma K_{\pi^*})^{-1}b\ge0.
\]
Let $\pi^+$ be a deterministic policy greedy with respect to $Q^+$.
The Bellman optimality inequality for $Q^*$ implies
\[
 Q^+-Q^*\le\gamma K_{\pi^+}(Q^+-Q^*)+b,
 \qquad
 Q^+-Q^*\le(I-\gamma K_{\pi^+})^{-1}b.
\]
Lemma~\ref{lem:universal-variance} therefore bounds either envelope
error by
\begin{equation}
 4\sqrt{2\eta L}\,H^{3/2}
 +2H\varepsilon\sqrt{\eta L}+8\eta H^2L.
 \label{eq:envelope-gap-estimate}
\end{equation}
Under \eqref{eq:envelope-step-condition}, this is at most
\[
 \varepsilon\left(\frac{4\sqrt2+2}{64}+\frac{8}{4096}\right)
 <\frac{\varepsilon}{2},
\]
where we used $\varepsilon\le1$ and $H\ge1$.
Finally, $m^+\ge0$ follows from $Q^+\ge Q^*\ge0$ and $r\ge0$;
and $m^-\le r+\gamma K_{\pi^*}Q^*=Q^*$ follows from $q^-\le Q^*$.
The lower envelope itself need not be nonnegative.
\end{proof}

We next state the comparison in a form that can be combined directly
with a bound on homogeneous propagation.  Along the observed trajectory,
write $X_t=(S_t,A_t)$ and $Z_t=(S_{t+1},\pi^*(S_{t+1}))$.
For the coordinate vector $e_x$, put
\begin{equation}
 B_t=I+\eta e_{X_t}(\gamma e_{Z_t}^{\mathsf T}-e_{X_t}^{\mathsf T}),
 \qquad
 \Phi_{t,0}=B_{t-1}\cdots B_0,
 \qquad \Phi_{0,0}=I.
 \label{eq:envelope-propagator}
\end{equation}
Each $B_t$ is a nonnegative substochastic matrix.  For such matrices,
$\|\Phi_{t,0}\|_\infty=\|\Phi_{t,0}\mathbf 1\|_\infty$.

\begin{lemma}[Two-sided comparison for zero-initialized Q-learning]
\label{lem:two-sided-envelopes}
Let $T\ge1$, $0<\delta<1$, $0<\varepsilon\le1$, and
\begin{equation}
 L=\log\frac{4d(T+1)}{\delta}.
 \label{eq:envelope-log}
\end{equation}
Suppose \eqref{eq:envelope-step-condition} holds.  Run ordinary
asynchronous Q-learning with the constant step size $\eta$ and $Q_0=0$:
\[
 Q_{t+1}(X_t)=(1-\eta)Q_t(X_t)
  +\eta\bigl[R_t+\gamma\max_a Q_t(S_{t+1},a)\bigr],
\]
leaving the other coordinates unchanged.  Assume each coordinate is
visited infinitely often almost surely, as holds for a finite,
irreducible behavior chain.  Then, with probability at least $1-\delta$,
simultaneously for all $0\le t\le T$,
\begin{equation}
 Q^*-\frac{\varepsilon}{2}\mathbf 1-H\Phi_{t,0}\mathbf 1
 \le Q_t\le Q^*+\frac{\varepsilon}{2}\mathbf 1.
 \label{eq:two-sided-comparison}
\end{equation}
In particular,
\begin{equation}
 \|Q_T-Q^*\|_\infty
 \le\frac{\varepsilon}{2}+H\|\Phi_{T,0}\|_\infty.
 \label{eq:envelope-final-bound}
\end{equation}
\end{lemma}

\begin{proof}
Construct $Q^+$ and $q^-$ as in Lemma~\ref{lem:fixed-envelopes}.
For each coordinate $x$, let $\tau_j(x)$ be its $j$th visit time,
starting with $j=1$.  Define the fixed targets
\begin{align*}
 Y_j^+(x)&=R_{\tau_j(x)}
    +\gamma\max_a Q^+(S_{\tau_j(x)+1},a),\\
 Y_j^-(x)&=R_{\tau_j(x)}
    +\gamma q^-(S_{\tau_j(x)+1},\pi^*(S_{\tau_j(x)+1})).
\end{align*}
Let $\mathcal F_t$ denote the history just before sampling the reward
and next state at time $t$, including $X_t$.
The conditional Markov assumption and the stopping-time property imply
\[
 \mathbb E[Y_j^\pm(x)\mid\mathcal F_{\tau_j(x)}]=m^\pm(x).
\]
The target revealed after one visit is measurable before the next
visit, so these are martingale differences on the visit clock.
By \eqref{eq:envelope-closeness}, their conditional standard deviations
are at most $\sigma_*(x)+\varepsilon/2$.  This follows from the
$L^2$ triangle inequality for centered random variables, since each
target differs pointwise from $R+\gamma V^*(S')$ by at most
$\gamma\varepsilon/2$.
Furthermore,
\[
 0\le Y_j^+(x)\le H+\gamma\varepsilon/2,
 \qquad
 -\gamma\varepsilon/2\le Y_j^-(x)\le H.
\]
Consequently, the absolute value of either centered target is at most
$2H$.

For each deterministic $k\in\{1,\ldots,T\}$, put
\[
 D_k^\pm(x)=\sum_{j=1}^k
       \eta(1-\eta)^{k-j}\bigl(Y_j^\pm(x)-m^\pm(x)\bigr).
\]
The weights are deterministic, their maximum is at most $\eta$, and
their squared sum is
\[
 \eta^2\sum_{j=1}^k(1-\eta)^{2(k-j)}
 \le\frac{\eta}{2-\eta}\le\eta.
\]
The one-sided Bernstein--Freedman inequality, applied separately to
$D_k^+$ and $-D_k^-$, therefore gives
\[
 D_k^+(x)\le b(x),\qquad D_k^-(x)\ge-b(x)
\]
except on an event of probability at most $2e^{-L}$.
For completeness, the corresponding Bernstein--Freedman threshold is
\[
 \sqrt{2\eta L}\bigl(\sigma_*(x)+\varepsilon/2\bigr)
 +\frac43\eta HL,
\]
which is bounded by \eqref{eq:envelope-bonus}.
A union bound over the $d$ coordinates and the $T$ deterministic
values of $k$ makes all these inequalities simultaneous with
probability at least $1-2dTe^{-L}\ge1-\delta$.
Denote this event by $\mathcal E$.  Controlling every deterministic
$k$ justifies substituting a random number of visits below; no
concentration inequality with future-dependent weights is used.

We first prove $Q_t\le Q^+$ on $\mathcal E$ by induction on $t$.
The claim holds at $t=0$.  If it holds before each of the first $k$
updates of a coordinate $x$, the monotonicity of the sampled Bellman
backup and expansion of that coordinate's recursion give
\[
 Q_t(x)
 \le\sum_{j=1}^k\eta(1-\eta)^{k-j}Y_j^+(x)
 =[1-(1-\eta)^k]m^+(x)+D_k^+(x)
 \le Q^+(x),
\]
where the last step uses $m^+(x)\ge0$.
Unvisited coordinates remain zero and also obey the inequality.
Thus the upper comparison holds at every time up to $T$.

For the lower comparison, run two auxiliary fixed-policy evaluation
recursions on the same samples:
\begin{align*}
 q_{t+1}&=B_tq_t+\eta R_t e_{X_t},&q_0&=0,\\
 w_{t+1}&=B_tw_t+\eta R_t e_{X_t},&w_0&=H\mathbf 1.
\end{align*}
Because the matrices $B_t$ are nonnegative and a maximum is at least
the value at $\pi^*$, induction gives $q_t\le Q_t$ pathwise.
Subtracting the two linear recursions gives the exact identity
\begin{equation}
 w_t-q_t=H\Phi_{t,0}\mathbf 1.
 \label{eq:auxiliary-difference}
\end{equation}
We claim that $w_t\ge q^-$ for every $t\le T$ on $\mathcal E$.
It holds initially since $q^-\le Q^*\le H\mathbf 1$.
If it has held before each of the first $k$ updates of coordinate $x$,
monotonicity and expansion give, with $a=(1-\eta)^k$,
\begin{align*}
 w_t(x)&\ge aH+(1-a)m^-(x)+D_k^-(x)\\
 &\ge aH+(1-a)m^-(x)-b(x)\\
 &=q^-(x)+a[H-m^-(x)]\ge q^-(x).
\end{align*}
The last inequality uses $m^-(x)\le H$, not nonnegativity of $q^-$.
Unvisited coordinates again satisfy the claim directly.
Combining this induction with \eqref{eq:auxiliary-difference} yields
\[
 Q_t\ge q_t=w_t-H\Phi_{t,0}\mathbf 1
 \ge q^--H\Phi_{t,0}\mathbf 1.
\]
Now \eqref{eq:envelope-closeness} proves
\eqref{eq:two-sided-comparison}, and nonnegativity of $\Phi_{t,0}$
proves \eqref{eq:envelope-final-bound}.
\end{proof}

\paragraph{Conditioning on a calibration stage.}
The same argument applies conditionally on any initial history when
the ensuing Q-learning run starts from zero and its step size is
chosen using that history and then held fixed.  In that application,
the envelopes are fixed conditional on the history.  The proof does
not require stationarity of the starting distribution.  Thus it can
be combined with a separate count-calibration argument, without
asserting a decreasing-step extension of the envelope lemma.

\section{Forgetting the Initialization}
\label{sec:initialization}
The initialization argument follows the reverse-particle representation
and hitting-cost compensation developed for TD learning by
\citet{peng2026td}. We adapt these arguments to propagation that selects
a fixed target action at the next state. This action projection is the
form of propagation left by the Bellman-envelope comparison.

Use the notation and assumptions of Section~\ref{sec:setup}, and fix a
deterministic target policy $\pi$, which may differ from the behavior
policy. The two transition kernels are
\[
 K_b((s,a),(v,c))=P(v\mid s,a)b(c\mid v),\qquad
 K_\pi((s,a),(v,c))=P(v\mid s,a)\mathbf 1_{\{c=\pi(v)\}}.
\]
Put $f_\pi(s,a)=(s,\pi(s))$. For a trajectory $X_t=(S_t,A_t)$,
the homogeneous propagation matrices with a constant step $\eta$ are
\[
 B_t=I+\eta e_{X_t}\bigl(\gamma e_{f_\pi(X_{t+1})}^{\mathsf T}
                   -e_{X_t}^{\mathsf T}\bigr),\qquad
 \Phi_{T,0}=B_{T-1}\cdots B_0.
\]
Every $B_t$ is nonnegative and substochastic when $0<\eta\le1$.

\subsection{Hitting times and a compensation identity}

For the reverse behavior chain define
\begin{align*}
 \tau_z&=\inf\{n\ge0:X_n=z\},&
 \tau_z^+&=\inf\{n\ge1:X_n=z\},\\
 h_z^*(x)&=\mathbb E_x^*\tau_z,&
 h_z^{*,+}(x)&=\mathbb E_x^*\tau_z^+,
 \qquad \mathfrak h_\star=\max_{x,z}h_z^*(x).
\end{align*}
Let $d_b(s)=\sum_a\mu(s,a)$ and
$\bar P_b(s,v)=\sum_a b(a\mid s)P(v\mid s,a)$.
Irreducibility implies $d_b(s)>0$, and stationarity gives
$\mu(s,a)=d_b(s)b(a\mid s)$. A bar over a hitting time will refer
to the physical-state chain $\bar P_b$, with zero-time conventions.
Define a potential on physical states and its lift to $\mathcal X$ by
\[
 a(s)=\sum_u d_b(u)\bar h_s^*(u),\qquad a(s,c)=a(s).
\]
The potential controls movement between physical states. The additional
waiting time for a specified action is bounded by a state--action return
time, giving the following compensation for the projected target kernel.

\begin{lemma}[Compensation for a fixed target policy]
\label{lem:offpolicy-compensation}
The potential satisfies $0\le a\le\mathfrak h_\star\mathbf1$.
For $x=(s,a)$ and $z=(v,c)$,
\begin{equation}
 h_z^{*,+}(x)\le \bar h_v^*(s)+\frac1{\mu(z)}.
 \label{eq:action-hitting}
\end{equation}
Consequently, with
$\ell_\pi(x)=\sum_z K_\pi(x,z)h_z^{*,+}(x)$,
\begin{equation}
 \ell_\pi\le(K_\pi-I)a+\frac2{\mu_{\min}}\mathbf1.
 \label{eq:offpolicy-compensation}
\end{equation}
\end{lemma}

\begin{proof}
The reverse transition formula is
\[
 K_b^*((s,a),(u,c))
 =\frac{d_b(u)b(c\mid u)P(s\mid u,c)}{d_b(s)}.
\]
Its right side does not depend on the current action $a$.
Summing over $c$ shows that the physical-state projection evolves
with $\bar P_b^*$. In particular, its maximal mean hitting time is
at most $\mathfrak h_\star$, which also proves the bound on $a$.
At physical state $v$, every reverse row is the same. Thus, from
$(v,a)$ with $a\ne c$, the mean time to hit $z=(v,c)$ equals the
positive return time from $z$, namely $1/\mu(z)$ by Kac's formula.
The same bound applies when the starting pair is $z$, using the
positive-time convention. Starting from a different physical state,
first hit $v$ and then apply the strong Markov property. This gives
\eqref{eq:action-hitting}.

The classical hitting-time reversal identity \citep{beveridge2010exit},
applied to the physical-state chain, gives
\begin{equation}
 \bar h_v^*(s)=\bar h_s(v)+a(v)-a(s).
 \label{eq:physical-reversal}
\end{equation}
A finite-dimensional derivation applies also to periodic chains.
For $P_0=\bar P_b$ and invariant column vector $d_0=d_b$,
the matrix $Z=(I-P_0+\mathbf1d_0^{\mathsf T})^{-1}$ exists.
The hitting equations give
$\bar h_v(s)=(Z_{vv}-Z_{sv})/d_0(v)$: this expression is zero
at $s=v$ and solves the hitting equation away from $v$.
The reverse fundamental matrix is
$Z^*=\operatorname{diag}(d_0)^{-1}Z^{\mathsf T}
\operatorname{diag}(d_0)$. Hence
$a(v)=Z_{vv}/d_0(v)-1$, and substituting these formulas proves
\eqref{eq:physical-reversal}.

The physical-state return equation gives
\[
 \sum_{a'}b(a'\mid s)\sum_vP(v\mid s,a')\bar h_s(v)
       =\frac1{d_b(s)}-1.
\]
Every summand is nonnegative. Therefore, for each $x=(s,a)$,
\[
 \sum_vP(v\mid s,a)\bar h_s(v)
 \le\frac{d_b(s)^{-1}-1}{b(a\mid s)}
 \le\frac1{\mu(x)}.
\]
Combining this inequality, \eqref{eq:action-hitting}, and
\eqref{eq:physical-reversal}, and using that $a$ depends only on
the physical state, proves \eqref{eq:offpolicy-compensation}.
\end{proof}

\subsection{A reverse particle and its lifetime}

Consider a reverse driver with kernel $K_b^*$. Given an alive
particle at $i$ and driver at $x$, first draw $y\sim K_b^*(x,\cdot)$.
If $y\ne i$, leave the particle at $i$. If $y=i$, independently
leave it at $i$ with probability $1-\eta$, move it to $f_\pi(x)$
with probability $\eta\gamma$, and kill it with probability
$\eta(1-\gamma)$. The next driver location is $y$.
Let $\zeta\ge1$ be its death time. The conditional particle masses
follow
\[
 p_{n+1}=p_n+\eta p_n(y)(\gamma e_{f_\pi(x)}-e_y).
\]
Consequently, conditional on a finite driver path, survival
probability is the sum of these masses. For a stationary forward
trajectory read backward, a particle initially at coordinate $i$
therefore represents the $i$th row sum of $\Phi_{T,0}$.

\begin{lemma}[Uniform expected lifetime]
\label{lem:particle-lifetime}
If $0<\eta H\le1$, then
\begin{equation}
 M:=\max_{i,x}\mathbb E_{i,x}\zeta
 \le4\left(\frac{H}{\eta\mu_{\min}}+\mathfrak h_\star\right).
 \label{eq:particle-lifetime}
\end{equation}
\end{lemma}

\begin{proof}
Write $F_i(x)=\mathbb E_{i,x}\zeta$, $D(i)=F_i(i)$,
and $r_\star=\max_{i,x}h_i^{*,+}(x)
\le\mathfrak h_\star+\mu_{\min}^{-1}$.
First $M$ is finite: in $\lceil2r_\star\rceil$ steps the driver
hits any fixed particle position with probability at least $1/2$;
at this first activation the particle dies with probability
$\eta/H$. Iterating this uniform conditional probability bound
over blocks gives a geometric tail.

Let $J$ be the driver state just before its first positive-time
visit to $i$. First-activation decomposition yields
\begin{equation}
 F_i(x)=h_i^{*,+}(x)+(1-\eta)D(i)
             +\eta\gamma\mathbb E_x^*[F_{f_\pi(J)}(i)].
 \label{eq:particle-first-activation}
\end{equation}
In particular,
\begin{equation}
 M\le r_\star+(1-\eta)D_{\max}+\eta\gamma M.
 \label{eq:particle-first-maximum}
\end{equation}
When the reverse driver starts at $i$, its last predecessor $J$
before return to $i$ has law $K_b(i,\cdot)$. To see this, reverse
every first-return excursion: the stationary factors in the
reverse transition probabilities telescope, so its probability
equals that of the reversed forward excursion. Its last
predecessor becomes the first forward successor. Almost-sure
return then gives the claimed law. Since the mean return time is
$1/\mu(i)$, \eqref{eq:particle-first-activation} on the diagonal gives
\[
 D(i)=\frac1{\eta\mu(i)}
          +\gamma\sum_z K_\pi(i,z)F_z(i).
\]
Apply \eqref{eq:particle-first-activation} once more, now to $F_z(i)$,
and put $\lambda=\gamma(1-\eta)$. We obtain
\[
 D\le\frac1{\eta\mu_{\min}}\mathbf1+
       \gamma\ell_\pi+\lambda K_\pi D+\eta\gamma^2M\mathbf1.
\]
The resolvent $(I-\lambda K_\pi)^{-1}$ is nonnegative and has row
sums $1/(1-\lambda)$, even if $K_\pi$ is reducible. Insert
\eqref{eq:offpolicy-compensation} before taking a maximum. The identity
\[
 (I-\lambda K_\pi)^{-1}(K_\pi-I)
 =\lambda^{-1}\bigl[(1-\lambda)(I-\lambda K_\pi)^{-1}-I\bigr]
\]
and $0\le a\le a_{\max}\mathbf1$ give
\[
 D_{\max}\le
 \frac{1/\eta+2\gamma}{(1-\lambda)\mu_{\min}}
 +\frac{a_{\max}}{1-\eta}
 +\frac{\eta\gamma^2}{1-\lambda}M.
\]
Substitute into \eqref{eq:particle-first-maximum}. The coefficient
of $M$ after multiplying by $1-\lambda$ is exactly
\[
 (1-\eta\gamma)(1-\lambda)
       -\eta\gamma^2(1-\eta)=1-\gamma=H^{-1}.
\]
It follows that
\[
 M\le H(1-\lambda)(r_\star+a_{\max})
       +\frac{H(1-\eta)(1/\eta+2\gamma)}{\mu_{\min}}.
\]
Now $H(1-\lambda)=1+\eta\gamma H\le2$,
$a_{\max}\le\mathfrak h_\star$, and
$(1-\eta)(1/\eta+2\gamma)\le2/\eta$. Thus
\[
 M\le4\mathfrak h_\star+2\mu_{\min}^{-1}
            +\frac{2H}{\eta\mu_{\min}}
 \le4\left(\mathfrak h_\star+\frac{H}{\eta\mu_{\min}}\right).
\]
\end{proof}

\begin{proposition}[Constant-step forgetting]
\label{prop:fixed-policy-forgetting}
Suppose $0<\eta H\le1$ and define
\[
 B=\left\lceil8\left(\frac{H}{\eta\mu_{\min}}
                                    +\mathfrak h_\star\right)\right\rceil.
\]
For every initial distribution of the behavior chain and every
$0<\rho\le1$,
\begin{equation}
 \mathbb P\bigl(\|\Phi_{T,0}\|_\infty>\rho\bigr)
 \le\frac{d}{\mu_{\min}\rho}\,2^{-\lfloor T/B\rfloor}.
 \label{eq:constant-forgetting}
\end{equation}
Here and below the matrix norm is the induced sup norm.
\end{proposition}

\begin{proof}
Markov's inequality and Lemma~\ref{lem:particle-lifetime} give
$\sup_{i,x}\mathbb P_{i,x}(\zeta>B)\le1/2$.
The strong Markov property at deterministic block boundaries then gives
$\sup_{i,x}\mathbb P_{i,x}(\zeta>T)\le2^{-\lfloor T/B\rfloor}$.
Under stationary time reversal, each expected row sum of $\Phi_{T,0}$
is bounded by this quantity. Markov's inequality and a union bound
over the $d$ rows give \eqref{eq:constant-forgetting} without its
factor $\mu_{\min}^{-1}$. For an arbitrary initial law $\nu$,
the likelihood ratio of any finite forward path relative to its
stationary law is $\nu(X_0)/\mu(X_0)\le\mu_{\min}^{-1}$.
This proves the stated bound.
\end{proof}

\subsection{Higher-power forgetting}

We adapt the power-comparison argument of \citet{peng2026td} to the
constant-step propagation studied here. The resulting bound permits
step selection independent of the confidence level. The integer $p$
is used only in the proof and does not enter the algorithm.

\begin{corollary}[Higher-power forgetting]
\label{cor:power-forgetting}
Let $p\ge1$ be an integer and suppose $pH\eta\le1/2$. Define
\[
 B_p=\left\lceil8\left(\frac{2H}{p\eta\mu_{\min}}
                               +\mathfrak h_\star\right)\right\rceil.
\]
For any initial distribution and $0<\rho\le1$,
\begin{equation}
 \mathbb P\bigl(\|\Phi_{T,0}\|_\infty>\rho\bigr)
 \le\frac{d}{\mu_{\min}\rho^p}\,2^{-\lfloor T/B_p\rfloor}.
 \label{eq:power-forgetting}
\end{equation}
The result also holds conditionally after any stopping-time pilot,
if the constant $\eta$ is chosen from the pilot history and then
held fixed throughout the displayed propagation interval.
\end{corollary}

\begin{proof}
Set
\[
 H_p=\frac{H+p-1}{p},\quad
 \gamma_p=\frac{H-1}{H+p-1},\quad
 c_p=1+\frac{p-1}{H},\quad
 \eta_p=1-(1-\eta)(1-\eta/H)^{p-1}.
\]
For $u,v\ge0$, weighted Jensen's inequality gives
\begin{align*}
 ((1-\eta)u+\eta\gamma v)^p
 &\le(1-\eta/H)^{p-1}
                    ((1-\eta)u^p+\eta\gamma v^p)\\
 &\le(1-\eta_p)u^p+\eta_p\gamma_p v^p.
\end{align*}
The coefficients of $u^p$ agree. The comparison of the other
coefficients is equivalent to
$(1-\eta/H)^{p-1}(1+(p-1)\eta/H)\le1$, which follows from
$(1-z)^{-(p-1)}\ge1+(p-1)z$ for $0\le z<1$.

Along the same trajectory, let $\Phi_{T,0}^{(p)}$ be the fixed-policy
propagator with discount $\gamma_p$ and step $\eta_p$.
Coordinatewise induction using the preceding scalar inequality yields
\[
 (\Phi_{T,0}\mathbf1)^{\odot p}
       \le\Phi_{T,0}^{(p)}\mathbf1,
 \qquad
 \|\Phi_{T,0}\|_\infty^p
       \le\|\Phi_{T,0}^{(p)}\|_\infty.
\]
Moreover,
\[
 \tfrac12c_p\eta\le\eta_p\le c_p\eta.
\]
The upper bound is the union bound for a product of numbers in
$[0,1]$. For the lower bound use
$(1-\eta)(1-\eta/H)^{p-1}\le e^{-c_p\eta}$ and
$1-e^{-z}\ge z/2$ for $0\le z\le1$; here
$c_p\eta\le p\eta\le1/(2H)<1$.
Since $H_p\le H$ and $c_p\le p$,
\[
 H_p\eta_p\le Hp\eta\le\tfrac12,
 \qquad
 \frac{H_p}{\eta_p\mu_{\min}}
       \le\frac{2H}{p\eta\mu_{\min}}.
\]
Apply Proposition~\ref{prop:fixed-policy-forgetting} to
$\Phi^{(p)}$ with threshold $\rho^p$; its block length is no
larger than $B_p$. This proves \eqref{eq:power-forgetting}.
For the conditional statement, condition on the pilot history.
The chosen step is then deterministic, and the strong Markov
property starts the continuation from its observed behavior state.
The bound is uniform over that state.
\end{proof}

\section{Lower Bounds and the Scope of Optimality}
\label{sec:lower-bounds}
This section distinguishes two sources of difficulty. An unvisited coordinate
cannot be learned by an algorithm that updates only the current coordinate.
The additional statistical cost of ordinary Q-learning is captured by the
constant-step lower bound of \citet[Theorem~5]{li2024q}. We also explain why
that lower bound applies when a preliminary trajectory is used to select a
step size, provided that Q-learning is subsequently initialized at zero and the
selected step size is held fixed. Our matching statement concerns mean-square
error and retains the parameter restrictions of the cited statistical lower
bound. It does not assert a lower bound for arbitrary online step size rules.

Write $H=(1-\gamma)^{-1}$. For $m>0$ and $h>0$, let
$\mathcal M(H,m,h)$ be the class of finite discounted MDPs, together with fixed
stationary behavior policies, for which the state-action chain is irreducible,
rewards belong to $[0,1]$, there are at most twelve admissible state-action
pairs, and
\[
  \min_x\mu_b(x)\ge m,
  \qquad
  \max_{x,z}\mathbb E_x^{K_b^*}\tau_z\le h.
\]
Here $\tau_z=\inf\{t\ge0:X_t=z\}$, and $K_b^*$ denotes the stationary time
reversal of the behavior chain. State-dependent finite action sets are allowed.

\subsection{An exploration lower bound for arbitrary step sizes}

\begin{proposition}[Unvisited-coordinate lower bound]
\label{prop:unvisited-lower}
Fix $H>1$, $0<m\le1/4$, and $h\ge1/m$. There is an MDP with two states and two
actions, with $\mu_{b,\min}=m$ and $\mathfrak h_\star=h$, such that zero-initialized
Q-learning satisfies
\begin{equation}\label{eq:hitting-lower-probability}
 \mathbb P\!\left(\|Q_T-Q^*\|_\infty=H\right)
 \ge \exp(-8T/h)
\end{equation}
for every integer $T\ge0$ and every predictable step size rule taking values in
$[0,1]$. The same assertion holds if the algorithm uses a preliminary
observation period and then initializes Q-learning at zero. Consequently,
\begin{equation}\label{eq:hitting-lower-risk}
 \mathbb E\|Q_T-Q^*\|_\infty^2
 \ge H^2\exp(-8T/h).
\end{equation}
\end{proposition}

\begin{proof}
Let the physical states be $\{0,1\}$. Both actions have reward one and share
the transition matrix
\[
 P_p=\begin{pmatrix}1-p&p\\p&1-p\end{pmatrix},
 \qquad
 p=\frac{1}{h-1/m+2}\in(0,1/2].
\]
The behavior policy chooses a designated rare action with probability
$q=2m$ at both states and the other action with probability $1-q$. Thus
$\mu_b(s,a)=b(a)/2$ and $\mu_{b,\min}=m$. Moreover,
\[
 K_b^*((s,a),(u,c))=P_p(s,u)b(c)=K_b((s,a),(u,c)).
\]
Fix a target coordinate $(v,c)$ whose action has behavior probability $r$.
Let $g_0$ be its expected hitting time from the other physical state and let
$g_1$ be its expected hitting time from physical state $v$ with the other
action. First-step analysis gives
\[
 g_0=1+(1-p)g_0+p(1-r)g_1,
 \qquad
 g_1=1+pg_0+(1-p)(1-r)g_1.
\]
Solving yields $g_1=2/r$ and $g_0=1/p+2/r-2$. Since $p\le1/2$ and
$q\le1/2$, the maximal hitting time is attained when $r=q$, and
\[
 \mathfrak h_\star=\frac1p+\frac2q-2=h.
\]

Start the physical chain at state zero and draw the initial action from the
behavior policy. All optimal action values equal $H$, whereas any coordinate
that has not been updated retains value zero. The iterates stay in $[0,H]$,
so one unvisited coordinate forces the sup-norm error to equal $H$.

If $h\ge2/m$, consider the event that the first $T$ transitions never leave
state zero. Its probability is $(1-p)^T$. The inequalities $p\le2/h$ and
$\log(1-p)\ge-2p$ imply
$(1-p)^T\ge\exp(-4T/h)$.
If $1/m\le h\le2/m$, instead consider the event that all of the first $T$
actions are the nonrare action. Behavior randomization is independent across
times conditional on the states, with the same probability at either state;
this event has probability $(1-q)^T$. Since $q\le1/2$ and $m\le2/h$,
\[
 (1-q)^T\ge\exp(-2qT)=\exp(-4mT)\ge\exp(-8T/h).
\]
Neither argument uses the step sizes. If a pilot phase precedes a reset, use
the same avoidance event over the entire $T$ observations; the relevant
coordinate also remains unvisited after the reset.
\end{proof}

In particular, achieving error smaller than $H$ with failure probability at
most $\delta$ requires $T\ge(h/8)\log(1/\delta)$ on this class. A mean-square
error guarantee $\mathbb E\|Q_T-Q^*\|_\infty^2\le\varepsilon^2$ requires
$T\ge(h/8)\log(H^2/\varepsilon^2)$.

\subsection{The statistical lower bound and calibration on a pilot trajectory}

We first state precisely the external result that supplies the statistical
obstruction. The form below uses sup-norm mean-square error, which is weaker
than the maximum coordinate mean-square lower bound in the original theorem.

\begin{lemma}[Constant-step obstruction; \citealp{li2024q}]
\label{lem:li-statistical-lower}
There are universal constants $c,C>0$ such that the following holds.
Fix an integer $n\ge2$ and suppose
\begin{equation}\label{eq:li-lower-regime}
 n\ge2,\qquad \gamma\ge0.95,\qquad
 0<m\le\frac{c}{\log^2 n},\qquad
 n\ge \frac{C H^7\log^3 n}{m}.
\end{equation}
There is an MDP with four states and at most three actions per state, and a
fixed behavior policy, with $\mu_{b,\min}=m$, for which zero-initialized
Q-learning satisfies
\begin{equation}\label{eq:li-supnorm-lower}
 \mathbb E_s\|Q_n-Q^*\|_\infty^2
 \ge \frac{cH^4}{mn\log^3 n}
\end{equation}
for every constant $\eta\in(0,1)$ and every initial physical state $s$, with
the first action drawn from the behavior policy. The same hard instance works
for all $\eta$ and $s$. Its reverse hitting time satisfies
\begin{equation}\label{eq:li-hitting-budget}
 \mathfrak h_\star\le\frac{H+1}{m}\le\frac{2H}{m}.
\end{equation}
\end{lemma}

\begin{proof}[Proof and source of the statistical bound]
The statistical inequality is \citet[Theorem~5]{li2024q}; its stated error
criterion is $\max_x\mathbb E|Q_n(x)-Q^*(x)|^2$, which implies
\eqref{eq:li-supnorm-lower}. Its hypotheses include zero initialization and
a constant step size, rather than a stationary initialization of the sample
chain. To make the parameterization explicit, write $a=1-\gamma$ and
\[
 \nu=(\nu_0,\nu_1,\nu_2,\nu_3)
 =\left(\frac25,\frac25-2m,\frac15,2m\right).
\]
The hard instance has action sets
$\mathcal A(0)=\mathcal A(2)=\mathcal A(3)=\{1,2\}$ and
$\mathcal A(1)=\{1,2,3\}$, hence nine admissible coordinates. Its rewards and
transition rows, in the state order $(0,1,2,3)$, are
\[
\begin{array}{c|c|l}
(s,u)&r(s,u)&P(\cdot\mid s,u)\\ \hline
(0,1)&2/3&(1,0,0,0)\\
(0,2)&0&(1-2a(1-\nu_0),2a\nu_1,2a\nu_2,2a\nu_3)\\
(1,1),(1,2)&1&(3a\nu_0/2,1-3a\nu_0/2,0,0)\\
(1,3)&0&(0,1-3a(\nu_2+\nu_3),3a\nu_2,3a\nu_3)\\
(2,1)&1&(2a\nu_0,0,1-2a\nu_0,0)\\
(2,2)&0&(0,2a\nu_1,1-2a(\nu_1+\nu_3),2a\nu_3)\\
(3,1)&3/4&(0,0,0,1)\\
(3,2)&0&(2a\nu_0,2a\nu_1,2a\nu_2,1-2a(1-\nu_3)).
\end{array}
\]
Take the behavior policy uniform on each state's action set. For sufficiently
small $m$, direct calculation gives $\mu_{b,\min}=\nu_3/2=m$. This is the
construction of \citet[Appendix~F.1]{li2024q} with
$c_\mu=2m\log^2 n$. We parameterize the rare-state mass directly so that
the actual minimum state--action occupancy is exactly $m$; absolute
normalization factors are absorbed into the universal constants of the
quoted theorem.

We record uniformity in the initial physical state because it is needed
below. The physical behavior kernel of this instance is exactly
\[
 \overline P_b=\gamma I+(1-\gamma)\mathbf1\nu^\top.
\]
For an arbitrary initial physical distribution $\lambda$ its time-$t$
distribution is exactly $\gamma^t\lambda+(1-\gamma^t)\nu$; see the calculation
preceding equation~(237). Thus all mixing bounds there are uniform in the
initial state. The visit-count bounds in Appendix~F.2.1 are likewise stated
for arbitrary initial states. Appendix~F.2.2 uses only the lower bound on the
terminal marginal after mixing, and Appendix~F.2.3 explicitly allows $n/3$
steps for mixing. None of these estimates restricts the initial state or
changes the hard instance as a function of the initial state or the
step size. This gives the uniform-start formulation above when the first
action is drawn from the fixed behavior policy, as in the original sampling
model.

For completeness, the hitting-time bound does not have to be extracted from
the cited proof. The physical kernel is reversible with invariant distribution
$\nu$. From a physical state different from $v$, its expected time to reach
$v$ is $H/\nu(v)$. In general, stationary reversal of a state-action behavior
chain has the form
\[
 K_b^*((s,a),(u,c))
 =\frac{\nu(u)b(c\mid u)P(s\mid u,c)}{\nu(s)},
\]
which is independent of $a$. Consequently, from any coordinate at physical
state $v$ other than a specified target $z=(v,c)$, the expected time to hit
$z$ equals the positive return time to $z$ when started at $z$: their
one-step transition distributions coincide. Kac's formula makes this time
$1/\mu_b(z)$. Applying the strong Markov property at the first visit to
physical state $v$ therefore gives
\[
 \mathbb E_x^{K_b^*}\tau_z
 \le\frac{H}{\nu(v)}+\frac1{\mu_b(z)}
 \le\frac{H+1}{m}.
\]
This proves \eqref{eq:li-hitting-budget}.
\end{proof}

\begin{proposition}[A step size selected from a pilot trajectory]
\label{prop:pilot-lower}
Fix deterministic integers $p\ge0$ and $n\ge2$. Observe the behavior
trajectory through physical state $S_p$, before drawing action $A_p$, and
choose any $\mathcal F_p$-measurable step size
$\widehat\eta\in(0,1)$, initialize a fresh Q-table at zero, and perform $n$
ordinary Q-learning updates with the frozen step size $\widehat\eta$ on the
continuation of the same trajectory. Under \eqref{eq:li-lower-regime}, the
hard instance in Lemma~\ref{lem:li-statistical-lower} satisfies
\[
 \mathbb E\|\widehat Q_{p+n}-Q^*\|_\infty^2
 \ge\frac{cH^4}{mn\log^3 n}.
\]
The choice of $\widehat\eta$ can use all observed rewards and transitions,
in addition to visit counts.
\end{proposition}

\begin{proof}
Specifically, take
\[
 \mathcal F_p=\sigma(S_0,A_0,R_0,S_1,\ldots,A_{p-1},R_{p-1},S_p).
\]
Thus the calibration phase does not expose $A_p$ or the outgoing training
transition. Conditional on
$\mathcal F_p$, the training table is zero, the selected step size is a fixed
number, and the training trajectory has the law of the same behavior chain
started at the observed physical state, with a fresh behavior action. The
Markov property and the
inequality uniform over initial states and constant step sizes
\eqref{eq:li-supnorm-lower} therefore give
\[
 \mathbb E\!\left[
 \|\widehat Q_{p+n}-Q^*\|_\infty^2\mid\mathcal F_p
 \right]
 \ge\frac{cH^4}{mn\log^3 n}
 \quad\text{almost surely}.
\]
Taking expectations proves the assertion. The sup-norm risk is used here
deliberately: the coordinate attaining a conditional maximum need not be
the same for different pilot realizations. A maximum coordinate mean-square
version follows with a further factor of at most twelve by summing the
coordinate risks before taking expectations.
\end{proof}

\subsection{A joint class statement and its scope}

\begin{theorem}[Joint lower bound for pilot-calibrated constant step sizes]
\label{thm:joint-lower}
Consider any algorithm of Proposition~\ref{prop:pilot-lower} with
$n\le T=p+n\le2n$. Suppose \eqref{eq:li-lower-regime} holds and
$h\ge2H/m$. Then, for universal constants $c,C>0$,
\begin{equation}\label{eq:joint-lower-risk}
 \sup_{M\in\mathcal M(H,m,h)}
 \mathbb E_M\|\widehat Q_T-Q_M^\star\|_\infty^2
 \ge c\left\{
 \frac{H^4}{mT\log^3 T}
 +H^2\exp(-8T/h)
 \right\}.
\end{equation}
The supremum includes the initial driver distribution. The statement also
applies to deterministic constant-step Q-learning by taking $p=0$.
\end{theorem}

\begin{proof}
The statistical instance of Lemma~\ref{lem:li-statistical-lower} belongs to
the class by \eqref{eq:li-hitting-budget}; apply
Proposition~\ref{prop:pilot-lower}. The exploration instance of
Proposition~\ref{prop:unvisited-lower} belongs to the same class and has
reverse hitting time exactly $h$. The supremum is at least the maximum of
the two lower bounds, hence at least half their sum. Finally $n\le T\le2n$
permits replacing $n$ by $T$, with a change of universal constants.
\end{proof}

The two hard instances need not coincide. They lie in the same parameter
class, which is sufficient for the supremum-risk assertion
\eqref{eq:joint-lower-risk}. The restriction $h\ge2H/m$ ensures that the
known statistical hard instance lies within the hitting-time budget. We do
not claim the same statistical lower bound for every smaller feasible
hitting-time budget.

Theorem~\ref{thm:joint-lower} identifies the $H^4/(mT)$ statistical rate and
an additive exploration cost, up to logarithmic factors, in its stated
large-sample regime. The restriction $n\gtrsim H^7/m$ is material: the
cited statistical theorem establishes a matching lower bound only in
this large-sample range. Nor does a mean-square lower bound alone imply a constant-failure-
probability statistical lower bound. High-probability upper bounds can be
compared to \eqref{eq:joint-lower-risk} by taking a sufficiently small failure
probability: since both the table and $Q^*$ belong to $[0,H]$, a bound
$\|\widehat Q_T-Q^*\|_\infty\le r_T(\delta)$ with probability at least
$1-\delta$ implies
\[
 \mathbb E\|\widehat Q_T-Q^*\|_\infty^2
 \le r_T(\delta)^2+H^2\delta.
\]
Choosing, for example, $\delta=(T+H)^{-4}$ changes logarithmic factors only
in the calibrated upper bound. The exploration lower bound itself is already
a probability statement and applies to arbitrary predictable step sizes.

\section{Discussion}
\label{sec:discussion}
Standard asynchronous Q-learning retains its quartic statistical horizon
dependence with an additive reverse hitting-time cost. The argument
combines fixed Bellman envelopes for the nonlinear update with
initialization analysis under a fixed optimal policy. Its reverse-time
compensation follows the approach used for TD learning by
\citet{peng2026td}, adapted to propagation through the target action at
the next state.

The empirical step-size rule attains this rate using a prescribed budget,
a preliminary count stage, and a zero-initialized training run with a
constant step. Extending the result to one uninterrupted recursion with
a decreasing step size would bring its algorithmic scope closer to the
TD guarantees. The fixed-envelope argument would need to control the
contribution of earlier, larger steps while preserving the additive
hitting-time cost. Extending the comparison beyond zero initialization
is another question.

The lower bounds identify the statistical and exploration costs for the
specified Q-learning procedures. The statistical result inherits the
large-sample restriction of \citet{li2024q}, and the joint bound requires
a hitting-time budget that contains their hard instance. Obtaining a
statistical lower bound for smaller hitting-time budgets, a broader
class of online step-size rules, or a fixed-confidence criterion would
clarify the scope of quartic horizon dependence further.

{\small\setlength{\bibsep}{2pt}
\bibliographystyle{abbrvnat}
\bibliography{ref}}

\begin{thebibliography}{8}
\providecommand{\natexlab}[1]{#1}
\providecommand{\url}[1]{\texttt{#1}}
\expandafter\ifx\csname urlstyle\endcsname\relax
  \providecommand{\doi}[1]{doi: #1}\else
  \providecommand{\doi}{doi: \begingroup \urlstyle{rm}\Url}\fi

\bibitem[Beveridge and Lov{\'a}sz(2010)]{beveridge2010exit}
A.~Beveridge and L.~Lov{\'a}sz.
\newblock Exit frequency matrices for finite {Markov} chains.
\newblock \emph{Combinatorics, Probability and Computing}, 19\penalty0
  (4):\penalty0 541--560, 2010.
\newblock \doi{10.1017/S0963548310000118}.

\bibitem[Lee(2026)]{lee2026sign}
D.~Lee.
\newblock Sign-separated asymmetric finite-time error analysis of {Q}-learning.
\newblock \emph{arXiv preprint arXiv:2605.16103}, 2026.
\newblock URL \url{https://arxiv.org/abs/2605.16103}.
\newblock Version 2.

\bibitem[Li et~al.(2022)Li, Wei, Chi, Gu, and Chen]{li2022async}
G.~Li, Y.~Wei, Y.~Chi, Y.~Gu, and Y.~Chen.
\newblock Sample complexity of asynchronous {Q}-learning: Sharper analysis and
  variance reduction.
\newblock \emph{IEEE Transactions on Information Theory}, 68\penalty0
  (1):\penalty0 448--473, 2022.
\newblock \doi{10.1109/TIT.2021.3120096}.
\newblock URL \url{https://arxiv.org/abs/2006.03041}.

\bibitem[Li et~al.(2024)Li, Cai, Chen, Wei, and Chi]{li2024q}
G.~Li, C.~Cai, Y.~Chen, Y.~Wei, and Y.~Chi.
\newblock Is {Q}-learning minimax optimal? {A} tight sample complexity
  analysis.
\newblock \emph{Operations Research}, 72\penalty0 (1):\penalty0 222--236, 2024.
\newblock \doi{10.1287/opre.2023.2450}.
\newblock URL \url{https://arxiv.org/abs/2102.06548}.

\bibitem[Peng(2026)]{peng2026td}
Y.~Peng.
\newblock Near-minimax-optimal rates for asynchronous {TD} learning with
  {Markovian} data.
\newblock \emph{arXiv preprint arXiv:2609.38880}, 2026.
\newblock URL \url{https://arxiv.org/abs/2609.38880}.

\bibitem[Qu and Wierman(2020)]{qu2020finite}
G.~Qu and A.~Wierman.
\newblock Finite-time analysis of asynchronous stochastic approximation and
  {Q}-learning.
\newblock In \emph{Proceedings of the Thirty Third Conference on Learning
  Theory}, volume 125 of \emph{Proceedings of Machine Learning Research}, pages
  3185--3205, 2020.
\newblock URL \url{https://proceedings.mlr.press/v125/qu20a.html}.

\bibitem[Tsitsiklis(1994)]{tsitsiklis1994asynchronous}
J.~N. Tsitsiklis.
\newblock Asynchronous stochastic approximation and {Q}-learning.
\newblock \emph{Machine Learning}, 16\penalty0 (3):\penalty0 185--202, 1994.
\newblock \doi{10.1007/BF00993306}.

\bibitem[Wainwright(2019)]{wainwright2019cone}
M.~J. Wainwright.
\newblock Stochastic approximation with cone-contractive operators: Sharp
  {$\ell_\infty$}-bounds for {$Q$}-learning.
\newblock \emph{arXiv preprint arXiv:1905.06265}, 2019.
\newblock URL \url{https://arxiv.org/abs/1905.06265}.

\end{thebibliography}

\clearpage
\appendix

\section{Occupancy Concentration from Hitting Times}
\label{app:coverage}
For the next lemma, $h_j(i)$ and $h_j^*(i)$ denote ordinary nonnegative
hitting times of $j$ from $i$ in the forward and reverse behavior
chains; in particular, both are zero when $i=j$.

\begin{lemma}[Occupancy concentration from maximum hitting time]
\label{lem:calibration_occupancy}
For every integer $n\ge1$ and every initial distribution,
\begin{equation}\label{eq:calibration_occupancy}
 \mathbb P\left\{\exists x:\ \frac{N_n(x)}n
                  \notin[\mu(x)/2,2\mu(x)]\right\}
 \le4d\exp\left\{-\frac{n}{96(1+\mathfrak h_\star)}\right\}.
\end{equation}
\end{lemma}

\begin{proof}
The hitting-time reversal identity gives
\[
 h_i(j)=h_j^*(i)+a(i)-a(j),\qquad
 a(i)=\sum_z\mu(z)h_i^*(z),\qquad
 0\le a(i)\le\mathfrak h_\star.
\]
For clarity, this identity also holds for periodic chains: with
$Z=(I-K_b+\mathbf1\mu^{\mathsf T})^{-1}$, the usual first-step
equations give $h_j(i)=(Z_{jj}-Z_{ij})/\mu(j)$; the reverse matrix is
$Z^*_{ij}=\mu(j)Z_{ji}/\mu(i)$, and substitution gives the displayed
identity. Consequently, $h_i(j)\le2\mathfrak h_\star$.

Put $F=1+2\mathfrak h_\star$ and fix a coordinate $x$. Let
$\tau=\inf\{t\ge0:X_t=x\}$, and let $\xi_1,\xi_2,\ldots$ be the
successive positive return times after this first visit. The strong
Markov property makes the $\xi_i$ independent and identically
distributed, with mean $m=1/\mu(x)$ by Kac's formula. Moreover,
\[
 \mathbb E\tau\le F,\qquad
 m=1+\sum_yK_b(x,y)h_x(y)\le F.
\]
For every integer $k\ge1$, the hitting and return times satisfy
\begin{equation}\label{eq:calibration_moments}
 \mathbb E\tau^k\le k!F^k,\qquad
 \mathbb E\xi_1^k\le k!mF^{k-1}.
\end{equation}
Indeed, for any of these times $Z$, on $\{Z>s\}$ its residual
waiting time at each integer $s\ge1$ is a first hitting time with
mean at most $F$; at $s=0$ it is the original waiting time. Apply
induction to
$Z^k\le k\sum_{s=0}^{Z-1}(Z-s)^{k-1}$ to get
$\mathbb EZ^k\le k!F^{k-1}\mathbb EZ$.
All the moments are finite because finite irreducibility gives a
uniform positive probability of a visit within a fixed finite block.

For $0\le\lambda<F^{-1}$, the moment bounds and
$e^{-z}\le1-z+z^2/2$ for $z\ge0$ imply
\begin{align*}
 \mathbb Ee^{\lambda\tau}&\le(1-\lambda F)^{-1},\\
 \log\mathbb Ee^{\lambda\xi_1}&\le\frac{m\lambda}{1-\lambda F},\\
 \log\mathbb Ee^{-\lambda\xi_1}&\le-m\lambda+mF\lambda^2.
\end{align*}
Set $\lambda=(4F)^{-1}$ and $W_k=\sum_{i=1}^k\xi_i$, with
$W_0=0$. For the lower occupancy tail, let
$k=\lceil n/(2m)\rceil-1$. The event $N_n(x)<n/(2m)$ implies
$\tau+W_k\ge n$, and $km<n/2$. Hence
\begin{align*}
 \mathbb P\{N_n(x)<n/(2m)\}
 &\le\mathbb P\{\tau\ge n/4\}
            +\mathbb P\{W_k\ge3n/4\}\\
 &\le\frac43e^{-n/(16F)}
       +\exp\left\{-\frac{3n}{16F}+\frac{km}{3F}\right\}
 \le\frac73e^{-n/(48F)}.
\end{align*}
This calculation also applies when $k=0$.
If $n\ge2m$, put $k=\lfloor2n/m\rfloor$. The event
$N_n(x)>2n/m$ implies $W_k<n$, while $km\ge3n/2$. Thus
\[
 \mathbb P\{N_n(x)>2n/m\}
 \le\exp\left\{\frac{n}{4F}-\frac{3km}{16F}\right\}
 \le e^{-n/(32F)}.
\]
Combining these bounds gives $4e^{-n/(48F)}$. When $n<2m\le2F$,
this last bound exceeds one and is still valid. A union bound and
$48F\le96(1+\mathfrak h_\star)$ prove the lemma.
\end{proof}

\section{Proof of the Constant-Step Theorem}
\label{app:main-proof}
\begin{proof}[Proof of Theorem~\ref{thm:constant-main}]
Let $\Phi_{T,0}$ be the homogeneous propagator of evaluation under $\pi^*$ along the behavior trajectory, defined in \eqref{eq:envelope-propagator}. Lemma~\ref{lem:two-sided-envelopes}, with failure probability $\delta/2$, applies because $L_T\ge\log(8d(T+1)/\delta)$ and its step condition is monotone in the logarithm. On that event,
\[
 \|Q_T-Q^*\|_\infty\le\varepsilon/2+H\|\Phi_{T,0}\|_\infty.
\]
Proposition~\ref{prop:fixed-policy-forgetting}, applied with $\rho=\varepsilon/(2H)$, makes the second term at most $\varepsilon/2$ with failure probability at most $\delta/2$, provided
\[
 T\ge
 \left\lceil8\left(\frac{H}{\eta\mu_{\min}}+\mathfrak h_\star\right)\right\rceil
 \left\lceil\log_2\frac{4dH}{\delta\varepsilon\mu_{\min}}\right\rceil.
\]
The step satisfies $\eta H\le1$, and \eqref{eq:constant-samples} implies this last condition for a sufficiently large $C$. A union bound completes the proof.
\end{proof}

\section{Q-Learning with an Empirically Calibrated Step Size}
\label{app:unknown-mu}
\begin{proof}[Proof of Theorem~\ref{thm:calibrated-upper}]
All constants below are universal. Increase $C_0$ so that every
logarithm appearing in the two-envelope bound with confidence
$\delta/3$ and length $n$ is at most $L_n$.
Write $E$ for the event in which all empirical frequencies belong
to the intervals in Lemma~\ref{lem:calibration_occupancy}.
The sample condition, for sufficiently large $C_1$, implies
\[
 n\ge96(1+\mathfrak h_\star)\log(12d/\delta),\qquad
 n\mu_{\min}\ge1.
\]
Therefore $\mathbb P(E^c)\le\delta/3$, and on $E$,
\[
 \tfrac12n\mu_{\min}\le M_n\le2n\mu_{\min}.
\]
The cap in \eqref{eq:calibration_step} is inactive when $C_1$ is
sufficiently large. In particular,
\begin{equation}\label{eq:calibration_bracket}
 \frac{AH\ell_n}{2\mu_{\min}n}
 \le\widehat\eta_n\le
 \frac{2AH\ell_n}{\mu_{\min}n}.
\end{equation}

Condition now on the pilot history through $S_n$, before drawing $A_n$. The
step $\widehat\eta_n$ and physical state $S_n$ are fixed under this
conditioning. Draw $A_n$ freshly from the behavior policy; subsequent
transitions have the original Markov reward law. All bounds used below hold uniformly over this
starting state and therefore apply conditionally. The event $E$ is
measurable with respect to the conditioning history.

Let $\Phi$ be the fixed-optimal-policy propagation matrix on the
estimation segment, and put
\[
 p=\left\lceil\frac{L_n}{\ell_n}\right\rceil,
 \qquad r=(n+8)^{-1},\qquad q=\delta/3.
\]
Because $L_n\ge\ell_n$, we have
$L_n\le p\ell_n\le2L_n$. By \eqref{eq:calibration_bracket},
\[
 pH\widehat\eta_n
 \le\frac{4AH^2L_n}{\mu_{\min}n}.
\]
Thus $pH\widehat\eta_n\le1/2$ when $C_1$ is sufficiently large.
Corollary~\ref{cor:power-forgetting} gives a block length satisfying
\[
 b_p=\left\lceil8\left\{\frac{2H}{p\widehat\eta_n\mu_{\min}}
                                  +\mathfrak h_\star\right\}\right\rceil
 \le\frac{32n}{AL_n}+8\mathfrak h_\star+1
\]
and the conditional failure bound
\[
 \mathbb P\{\|\Phi\|_{\infty\to\infty}>r\mid\mathcal F_n\}
 \le\frac d{\mu_{\min}}r^{-p}
                          2^{-\lfloor n/b_p\rfloor}.
\]
Define
\[
 k=\left\lceil\log_2\frac d{\mu_{\min}q r^p}\right\rceil.
\]
The inequalities $\log(d/(\mu_{\min}q))\le L_n$ and
$p\log(n+8)\le L_n$ show that $k\le4L_n$. Consequently,
\[
 b_pk\le\frac{128}{A}n
              +32\mathfrak h_\star L_n+4L_n\le n,
\]
by the specified value of $A$ and a sufficiently large $C_1$.
We obtain
\begin{equation}\label{eq:calibration_forgetting}
 \mathbb P\left(E\cap\{\|\Phi\|_{\infty\to\infty}
                               >(n+8)^{-1}\}\right)\le\delta/3.
\end{equation}
The integer $p$ is only a proof parameter and is not used by the
algorithm.

The upper bracket in \eqref{eq:calibration_bracket} and
\eqref{eq:calibrated_sample} also imply
\[
 \widehat\eta_n\le\frac{c\varepsilon^2}{H^3L_n},
 \qquad \widehat\eta_nH\le1,
\]
where $c$ can be the numerical constant in the two-envelope lemma.
Apply Lemma~\ref{lem:two-sided-envelopes} conditionally to the zero-initialized
estimation run, with confidence $\delta/3$. On $E$, outside an
additional event of unconditional probability at most $\delta/3$,
\[
 \|\widehat Q_{2n}-Q^*\|_\infty
 \le\varepsilon/2+H\|\Phi\|_{\infty\to\infty}.
\]
Finally, \eqref{eq:calibrated_sample} ensures $H/(n+8)\le\varepsilon/2$.
Combining the occupancy, envelope, and forgetting events proves
the asserted probability bound.
\end{proof}

\begin{proof}[Proof of Corollary~\ref{cor:calibrated_mse}]
Apply Theorem~\ref{thm:calibrated-upper} with accuracy $\varepsilon/2$
and confidence $\delta=\varepsilon^2/(4H^2)$. Both the Q-learning
iterates and $Q^*$ belong to $[0,H]^d$, so the squared error is
at most $H^2$ on the failure event. Its expectation is therefore at
most $\varepsilon^2/4+\delta H^2\le\varepsilon^2$.
Neither this accuracy nor this confidence is an algorithm input.
\end{proof}

\end{document}